\documentclass{article}

\usepackage{xarxiv}

 \usepackage[british]{babel}

\usepackage{natbib} 

\usepackage{mathtools} 
\usepackage{booktabs} 
\usepackage{tikz} 

\usepackage{graphicx}
\usepackage{multirow}
\usepackage{todonotes}
\usepackage{stmaryrd,nicefrac}
\usepackage{amsmath}
\usepackage{amssymb}
\usepackage{hyperref}
\usepackage{subfig}
\usepackage{amsthm}

\usepackage{hyperref}

\usepackage{amsmath}
\usepackage{amssymb}
\usepackage{mathtools}
\usepackage{amsthm}

\usepackage[capitalize,noabbrev]{cleveref}

\usepackage{enumitem,nicefrac}

\theoremstyle{plain}
\newtheorem{theorem}{Theorem}[section]
\newtheorem{proposition}[theorem]{Proposition}
\newtheorem{lemma}[theorem]{Lemma}

\theoremstyle{definition}
\newtheorem{definition}[theorem]{Definition}

\theoremstyle{remark}

\newtheorem{example}[theorem]{Example}

\renewcommand{\vec}[1]{\boldsymbol{#1}}
\newcommand{\given}{\, | \,}
\newcommand{\hath}{\hat{h}}

\newcommand{\R}{\mathbb{R}}
\newcommand{\E}{\mathbb{E}}

\newcommand{\fromto}{\longrightarrow}

\newcommand*{\defeq}{\mathrel{\vcenter{\baselineskip0.5ex \lineskiplimit0pt
			\hbox{\footnotesize.}\hbox{\footnotesize.}}}%
	=}

\newcommand{\cX}{\mathcal{X}}
\newcommand{\cY}{\mathcal{Y}}
\newcommand{\cH}{\mathcal{H}}
\newcommand{\cM}{\mathcal{M}}
\newcommand{\cD}{\mathcal{D}}

\newcommand{\prob}{p}

\newcommand{\argmin}{\operatorname*{argmin}}

\title{\large On Second-Order Scoring Rules for Epistemic Uncertainty Quantification}

\author{
  Viktor Bengs, \, Eyke H\"ullermeier\\
  Institute of Informatics, 
  University of Munich (LMU)\\
  Munich Center for Machine Learning\\
  \texttt{viktor.bengs@lmu.de, eyke@lmu.de} 
  \AND
Willem Waegeman\\
Department of Data Analysis and Mathematical Modeling\\
Ghent University\\
  \texttt{Willem.Waegeman@UGent.be}
}

\begin{document}
	
	\maketitle

\begin{abstract}
It is well known that accurate probabilistic predictors can be trained through empirical risk minimisation with proper scoring rules as loss functions. While such learners capture so-called aleatoric uncertainty of predictions, various machine learning methods have recently been developed with the goal to let the learner also represent its epistemic uncertainty, i.e., the uncertainty caused by a lack of knowledge and data. An emerging branch of the literature proposes the use of a second-order learner that provides predictions in terms of distributions on probability distributions. However, recent work has revealed serious theoretical shortcomings for second-order predictors based on loss minimisation. In this paper, we generalise these findings and prove a more fundamental result: There seems to be no loss function that provides an incentive for a second-order learner to faithfully represent its epistemic uncertainty in the same manner as proper scoring rules do for standard (first-order) learners. As a main mathematical tool to prove this result, we introduce the generalised notion of second-order scoring rules.
\end{abstract}

\section{Introduction}

The representation and quantification of uncertainty in machine learning, most notably of predictive uncertainty in the setting of supervised learning, has recently attracted increasing attention \citep{mpub440}. Going beyond standard probabilistic prediction, various methods have been proposed that seek to distinguish between so-called aleatoric and epistemic uncertainty \citep{mpub272,kend_wu17}. One way to do so it to learn second-order predictors $H: \mathcal{X} \fromto \mathbb{P}(\mathbb{P}(\mathcal{Y}))$ mapping a query instance $\vec{x}$ to a probability distribution on the probability distributions over the outcome space $\mathcal{Y}$. This is motivated as follows: Assuming that outcomes cannot be predicted deterministically, and hence taking a conditional probability $p^* = p^*(\cdot \given \vec{x})$ on $\mathcal{Y}$ as ground-truth, it is natural to train a probabilistic predictor producing estimates $\hat{p} = \hat{p}(\cdot \given \vec{x})$. Such estimates capture aleatoric uncertainty about the actual outcome $y \in \mathcal{Y}$, i.e., inherent randomness that the learner cannot get rid of (even with perfect knowledge about $p^*$, the outcome $y$ remains random to some extent). However, it does not allow the learner to express its epistemic uncertainty, namely, its lack of knowledge about how accurately $\hat{p}$ approximates $p^*$. To capture this uncertainty as well, the learner is allowed to predict a second-order distribution $Q$. In other words, instead of committing to a single (point) prediction $\hat{p}$, the learner assigns probabilities $Q(p)$ to all candidate distributions $p^*$.

How to train a second-order predictor $H$ on empirical data in the form of tuples $(\vec{x}, y) \in \mathcal{X} \times \mathcal{Y}$, as commonly assumed in supervised learning? To this end, several authors have proposed extensions of empirical risk minimisation, i.e., to find a predictor $H$ that minimises the (regularised) loss on the training data. Obviously, this requires a second-order loss function $L_2$ that compares second-order predictions with actual outcomes: $L_2(Q , y)$ is the loss suffered by the learner when predicting $Q = H(\vec{x})$ and observing outcome $y$. Different loss functions of this kind have been proposed for classification \citep{sens_ed18,MalininG18,MalininG19,MalininMG20, char_ue20,HuseljicSHK20,KopetzkiCZGG21,tsiligkaridis2021information,Bao2021EvidentialDL,hammam2022predictive} and regression \citep{amini2020deep,Ma2021TrustworthyMR,MalininUnpub,Charpentier2022NaturalPN,Oh_Shin_2022,Shankar2023AAAI}.

Focusing on the classification setting, \citet{bengs2022pitfalls} have shown theoretical shortcomings of second-order loss minimization. In particular, they prove that the second-order loss functions proposed in the literature do not incentivise the learner to predict its epistemic uncertainty in a faithful way. Similar issues have been revealed by \citet{meinert2022unreasonable} for empirical loss minimisation in the regression setting. While criticising specific types of losses, none of these papers strictly excludes the existence of other loss functions that may provide the right incentive for the learner.

In this paper, we therefore strive for a more general result, which applies to any kind of loss function and to both the classification and regression setting. To this end, we introduce second-order scoring rules as our main mathematical tool. For the case of standard (first-order) probabilistic predictions, it is well known that loss functions in the form of proper scoring rules (such as log-loss in classification and squared error loss in regression) provide exactly the right incentive to the learner: To minimise such a loss in expectation, the learner has to provide unbiased predictions of ground-truth probabilities $p^*(\cdot \given \vec{x})$. 

Transferring this notion from the aleatoric to the epistemic level, we ask the following question: Is there a second-order loss $L_2$ that incentivises the learner to be honest in the sense of predicting $Q = H(\vec{x})$ whenever $Q$ corresponds to its actual belief about the ground-truth distribution $p^*(\cdot \given \vec{x})$? Our main result is again a negative answer to this question: There seems to be no meaningful second-order loss function (scoring rule) incentivising the learner to faithfully reveal its true beliefs on the epistemic level.

\section{Setting and Notation}
We assume a standard supervised learning setting with instance (or feature) space $\cX$, label (or outcome) space $\cY $,  and training data 
$
\cD = \big\{ \big(\vec{x}^{(n)} , y^{(n)} \big) \big\}_{n=1}^N \subset \cX \times \cY \, .
$
Here, $\cY$ can either correspond to a classification task (i.e., $\cY= \{ y_1 , \ldots , y_K \}$ for some $K\in \mathbb{N}_{\geq 2}$) or a regression task (i.e., $\cY= \mathbb{R}$).
Following the classical setting, we also assume that the data is generated i.i.d.\ according to an underlying joint probability $p^*$ on $\cX \times \cY$, i.e., each $z^{(n)}= (\vec{x}^{(n)} , y^{(n)} )$ is a realisation of $Z=(X,Y) \sim p^*$.
Correspondingly, each instance $\vec{x} \in \cX$ is associated with a conditional distribution $p^*( \cdot \given \vec{x})$ on $\cY$, such that $p^*( y \given \vec{x})$ is the probability to observe label $y$ as an outcome given $\vec{x}$. 

Let $\mathbb{P} (\Omega)$ denote the set of probability distributions on the measurable space $(\Omega,\mathcal{A}),$ where $\mathcal{A}$ is a $\sigma$-algebra on $\Omega.$
We write $\mathbb{P}_1(\cY) := \mathbb{P} (\cY)$ for the set of all probability distributions over $\cY$ and $\mathbb{P}_2(\cY) := \mathbb{P}(\mathbb{P} (\cY))$ for the set of all probability distributions over $\mathbb{P}(\cY).$
For sake of convenience, we also define $\mathbb{P}_0(\cY) := \cY.$
We refer to the elements in $\mathbb{P}_1(\cY)$ as first-order distributions, while the elements in $\mathbb{P}_2(\cY)$ are referred to as second-order distributions.
We shall use lowercase letters, e.g.\ $\hat p, p,$ for elements of the former, and uppercase letters, e.g.\ $\hat Q,Q,$ for elements of the latter.
The Dirac measure at $y \in \mathbb{P}_0(\cY)$ is denoted by $\delta_y \in \mathbb{P}_1(\cY)$; likewise, $\delta_p \in \mathbb{P}_2(\cY)$ denotes the Dirac measure at $p \in \mathbb{P}_1(\cY),$ where the underlying space of the Dirac measure should be clear from the context.
Finally, we write $\overline{\R} = [-\infty,\infty]$ for the extended real line.

\paragraph{Learning Predictive First-Order Models}

Suppose a \emph{hypothesis space} $\mathcal{H}_1 \subset \mathbb{P}_1(\cY)^{\cX}=\{ h:\cX \to \mathbb{P}_1(\cY)  \}$ to be given. 
Thus, a hypothesis $h$ in $\mathcal{H}_1$ maps instances $\vec{x}\in\cX$ to probability distributions on outcomes (first-order distributions).
In standard supervised learning, the goal of the learner is to induce a hypothesis (predictive model) with low (first-order) risk 
\begin{equation} \label{eq:risk}
	R_1(h) \defeq \int\limits_{\cX \times \cY} L_1( h(\vec{x}) , y) \, \mathrm{d}p^*(\vec{x} , y) \enspace ,
\end{equation}
where $L_1: \, \mathbb{P}_1 ( \cY) \times \mathcal{Y} \longrightarrow \mathbb{R}$ is a (first-order) loss function.  
The choice of a hypothesis is commonly guided by the empirical risk  
\begin{equation}\label{eq:erisk}
	R_{1,emp}(h) \defeq  N^{-1} \sum\limits_{n=1}^N L_1\left(h(\vec{x}^{(n)}) , y^{(n)}\right) \enspace ,
\end{equation}
i.e., the performance of a hypothesis on the training data. However, since $R_{1,emp}(h)$ is only an estimation of the true risk $R_1(h)$, the empirical risk minimiser 
$$
\hath \defeq \argmin_{h \in \cH_1} R_{1,emp}(h)
$$ 
(or any other predictor) favored by the learner will normally not coincide with the true risk minimiser (Bayes predictor)
$h^* \defeq \argmin_{h \in \cH_1} R_1(h)$. 
Correspondingly, there remains (epistemic) uncertainty regarding $h^{*}$ as well as the approximation quality of $\hath$ (in the sense of its proximity to $h^*$) and the predictions
$\hat{p}(\cdot \given \vec{x}) = \hat{h}(\vec{x})$ 
produced by this hypothesis.

\medskip

\begin{example}
	In the classification setting with $\cY= \{ y_1 , \ldots , y_K \}$, appropriate first-order loss functions are the Brier score or the cross-entropy loss:
	\begin{align}
		L_1^{\texttt{Brier}}(p,y) &= \sum_{k=1}^K (p(y_k) - 1_{ \{y_k = y\} })^2 \, , \label{def:Brier_score} \\
		L_1^{\texttt{CE}}(p,y) &= - \sum_{k=1}^K 1_{\{ y_k = y\} } \log(p(y_k)) \, , \label{def:cross_entropy_loss}
	\end{align}
        where $1_{ \{ \cdot\} }$ is the indicator function.
	Both have the appealing property that the optimal hypothesis $h^*(\vec{x})$ coincides with the conditional class distribution $p^*(\cdot \given \vec{x}).$
	Other suitable first-order losses include the spherical score, Winkler's score, or the Beta score. We refer to \citet{gneiting2007strictly} for an overview, who also provide examples of appropriate first-order loss functions for the case of regression (i.e., $\cY=\R$). 
\end{example}


\paragraph{Learning Predictive Second-Order Models}

Quite recently, there has been much interest in predictive models of second-order, i.e., mappings from instances $\vec{x}\in\cX$ to probability distributions on probability distributions over the outcomes (second-order distributions).
Formally, a hypothesis space  $\mathcal{H}_2 \subset \mathbb{P}_2(\cY)^{\cX}=\{ H:\cX \to \mathbb{P}_2(\cY)  \}$ is considered. 
Thus, $H(\vec{x})$ assigns a probability to each distribution $p(\cdot \given \vec{x}) \in \mathbb{P}_1(\cY)$, and the more certain the learner about the true conditional distribution $p^*(\cdot \given \vec{x})$, the more concentrated or ``peaked'' $H(\vec{x})$ is. 

	In light of this, the basic idea of \emph{direct} epistemic uncertainty prediction  is to try to learn a second-order predictor in the ``classical'' way through loss minimisation, just like a first-order predictor. 
	Formally, a second-order loss function
	\begin{equation}
		L_2:  \mathbb{P}_2(\cY) \times \cY \fromto \R
	\end{equation}
	is specified, which compares second-order predictions $H(\vec{x})$ with (zero-order) observations $y$, such that minimising $L_2$ on the training data $\cD$ yields a ``good'' second-order predictor.
	Formally, the minimiser $\hat H$ of the empirical risk induced by $L_2,$ i.e., 
	\begin{equation}\label{eq:erisk_level_2}
		R_{2,emp}(H) \defeq  N^{-1} \sum\limits_{n=1}^N L_2\left(H(\vec{x}^{(n)}) , y^{(n)}\right) \enspace ,
	\end{equation}
	over the considered hypothesis space $\mathcal{H}_2$ should provide accurate predictions for a given $\vec{x}$ by means of $\E_{p \sim \hat H(\vec{x})} \E_{Y \sim p}[Y],$ while reporting the second-order (epistemic) uncertainty in a reasonable and faithful manner.	
	Preferably, these properties should be reflected by the true risk minimiser 
	$H^* \defeq \argmin_{H \in \cH_2} R_2(H),$ where 
	\begin{equation} \label{eq:risk_level_2}
		R_2(H) \defeq \int\limits_{\cX \times \cY} L_2( H(\vec{x}) , y) \, \mathrm{d}p^*(\vec{x} , y)  
	\end{equation}
	is the (second-order) risk induced by the loss $L_2.$

	Usually, the hypotheses in $\cH_2$ are mappings from $\cX$ to a parameterized family of second-order distributions.
	More precisely, the image of these mappings is $\mathbb{P}_2(\cM),$ where $\cM$ is some specific parameter space  such that $\mathbb{P}_2(\cM)$ is in fact a strict subset of $\mathbb{P}_2(\cY).$
	Each element $Q \in  \mathbb{P}_2(\cM)$ can be encoded by means of a parameter vector $\mathbf{m} \in \cM,$ i.e., $Q = Q_{\mathbf{m}}.$
	Thus, the hypotheses in $\cH_2$ are encoded by a mapping from instances $\vec{x} \in \cX$ to a parameter vector $\mathbf{m}.$
	In light of this, the second-order distributions $Q \in  \mathbb{P}_2(\cM)$ have in most cases support only on a strict subset of $\mathbb{P}_1(\cY)$  due to the parameterization.
	This support is again usually a parameterized family of first-order distributions $\mathbb{P}_1(\Theta),$ where $\Theta$ is yet another parameter space.

\medskip
	
	\begin{example}[Classification]
		Consider the classification setting with $\cY=\{y_1,\ldots,y_K\}$, where the goal is to learn a predictive second-order model.
	Here, the most commonly used parameterized class of second-order distributions $\mathbb{P}_2(\cM)$ is the set of Dirichlet distributions with parameter space 
 $$\cM = \big\{ \mathbf{m}=(m_1,\ldots,m_K) \, | \, m_i >0, \ i=1,\ldots,K  \big\}$$ 
 having support on the (first-order) categorical distributions $\mathbb{P}_1(\Theta),$ where 
 $$\Theta= \big\{ \vec{\theta} = (\theta_1, \ldots , \theta_K) \in [0,1]^K \, | \, \| \vec{\theta} \|_1 = 1 \big\}$$ 
    (see \citep{sens_ed18,MalininG18,MalininG19,MalininMG20, char_ue20,HuseljicSHK20,KopetzkiCZGG21,tsiligkaridis2021information,Bao2021EvidentialDL,hammam2022predictive}).
		In this case, $\mathbb{P}_1(\cY) = \mathbb{P}_1(\Theta).$ 

	In this setting, losses of the following kind have been suggested:
	\begin{align} \label{def:bayes_losses}
		\begin{split}
			L^{\texttt{Bay}}_2\big( Q,y \big) 
			&= \E_{p \sim Q} \, L_1 \big(p, y \big)  + \lambda \, d_{KL}(Q,Q_0) \, ,
		\end{split}
	\end{align}
	where  $\lambda \geq 0$ is some regularisation parameter,	$ L_1: \mathbb{P}(\cY) \times \cY \to \R$  is some appropriate first-order loss, $ d_{KL}: \mathbb{P}_2(\cY) \times  \mathbb{P}_2(\cY) \to \R$  the KL-divergence, and  $Q_0 $ the uniform distribution on $\mathbb{P}_2(\cY).$
	The idea is that the first component in \eqref{def:bayes_losses} enforces correct predictions, which, however, might favor peaked second-order distributions. 
	Therefore, the second component in \eqref{def:bayes_losses} acts as a countermeasure, since it penalizes deviations from the most non-peaked second-order distribution, namely the uniform distribution.
	\end{example}
	\begin{example}[Regression] \label{example:Regression}
	Consider the regression setting, i.e., $\cY=\R,$  where again the goal is to learn a predictive second-order model.
\citet{amini2020deep} published the pioneering work in this regard, using normal-inverse gamma (NIG) distributions for the parameterized class of second-order distributions $\mathbb{P}_2(\cM)$, such that 
	$$\cM = \big\{ \mathbf{m}=(m_1,\ldots,m_4) \, | \, m_1 \in \R, m_2,m_3-1,m_4>0 \big\} \, .$$
	Note that \citet{amini2020deep} denote $(m_1,m_2,m_3,m_4)$  by $(\gamma,\nu,\alpha,\beta)$.
	Accordingly, this second-order distribution is essentially a distribution over the set of Gaussian distributions, i.e., 
	$$ \mathbb{P}_1(\Theta) = \{ \mathrm{N}(\mu,\sigma^2) \, | \, (\mu,\sigma) \in \Theta  \}, $$
	where $\Theta = \{ (\mu,\sigma) \, | \, \mu\in \R, \sigma > 0 \}$ is the set of (location-scale) parameters of Gaussian distributions.

	The second-order loss function suggested in this regard is 
	\begin{align}\label{def:DER_loss}
		L^{\texttt{DER}}\big( \mathbf{m},y \big) 
		&= L^{\mathrm{t}}\big( \mathbf{m},y \big) + \lambda \cdot \mathrm{PEN}\big( \mathbf{m},y \big),
	\end{align}
	where
	\begin{align} 
	\begin{split}
			L^{\mathrm{t}}\big( \mathbf{m},y \big)  
			%
			&= \nicefrac12 \log\left( \nicefrac{\pi}{m_2}\right) 
			 - m_3 \log(m_{2,4})  \label{def:neg_log_likeli_student} \\
			&\quad \ + \left(m_3+\nicefrac{1}{2}\right) \log\left((y-m_1)^2m_2 + m_{2,4}\right)  \\
			& \quad + \log\left( \nicefrac{\Gamma(m_3)}{\Gamma(m_3+\frac{1}{2})}\right), \\
             m_{2,4} &:= 2m_4(1+m_2), \\
          \mbox{and \ }   \mathrm{PEN}\big( \mathbf{m}, y \big) &= |m_1 - y|\cdot (m_3 + 2 m_2 ).
		\end{split} 
	\end{align}
	Here, $L^{\mathrm{t}}$ is the negative log-likelihood function of a Student-t distribution with location parameter $m_1$, scale parameter $ 2m_3$ and $ \frac{m_{2,4}}{2 m_2 m_3}$ degrees of freedom.
	Similarly, as for the class of second-order loss functions in the classification setting in \eqref{def:bayes_losses}, the first component in \eqref{def:DER_loss} should enforce correct predictions, while the second component prevents the use of too peaked second-order distributions.
	For an NIG distribution, this can be achieved by penalizing too large values of $m_2$ and $m_3,$ as the variance terms of an NIG distribution depend reciprocally on these, respectively.

Follow-up papers on \citep{amini2020deep} adjust the loss in (\ref{def:DER_loss}) by replacing the negative log-likelihood with squared loss \citep{Oh_Shin_2022}, by changing the regularization term \citep{Shankar2023AAAI}, or by considering a mixture of NIG distributions instead of a single one \citep{Ma2021TrustworthyMR}. Another line of research modifies  (\ref{def:bayes_losses}) to the regression setting \citep{MalininUnpub,Charpentier2022NaturalPN}. 

	\end{example}
	%

\section{Scoring Rules for First-order Losses} \label{section:level-1-scoring-rules}

In this section, we review the essential concepts of proper scoring rules, which is a class of (first-order) loss functions incentivizing the learner to predict probabilities in an unbiased way.
Here, unbiased means that the learner minimises expected loss if (and only if) it predicts the true (conditional) probability distribution.

\begin{definition} \label{def:level_1_scoring_rule}
	 A (first-order) scoring rule $	S_1:   \mathbb{P}_1(\cY) \times  \mathbb{P}_1(\cY)  \to \overline{\R}$ based on the (first-order) loss $L_1: \mathbb{P}_1(\cY) \times \cY \to \overline{\R},$ such that $L_1(p,\cdot)$ is $\mathbb{P}_1(\cY)$-quasi-integrable\footnote{A function defined on $\cY$ and taking values in the extended real line is $ \mathbb{P}_1(\cY)$-quasi-integrable if it is measurable w.r.t.\ $\mathcal{A}$ and is quasi-integrable w.r.t.\ all $p \in \mathbb{P}_1(\cY)$.} for all $p \in \mathbb{P}_1(\cY)$,  is given for all $\hat p,p \in \mathbb{P}_1(\cY)$ by 
	 \begin{align} \label{def_level_1_score}
	 	S_1(\hat p, p) = \E_{Y \sim p}[ L_1(\hat p, Y) ].
	 \end{align}  
\end{definition}
The second component (i.e., $p$) of a scoring rule represents the target distribution or ground-truth $p^*( \cdot  \given \vec{x})$, while the first component  (i.e., $\hat p$) represents the predicted distribution, e.g.\ $\hat{p}(\cdot \given \vec{x}) = \hat{h}(\vec{x})$.
Thus, integrating \eqref{def_level_1_score} over the distribution of the instances $\vec{x}$ leads to the (first-order) risk in \eqref{eq:risk}, i.e., 
$R_1(\hat{h})= \int_\cX S_1( \hat{h}(\vec{x}), p(\vec{x}|y))\,  \mathrm{d}p_{X}(\vec{x}),$
where $p_{X}$  denotes the distribution over the instances.

Note that in the literature it is more common to refer to the loss function $L_1$ as the scoring rule, while $S_1$ is referred to as the expected score \citep{gneiting2007strictly,ovcharov2018proper}. 
However, to make the distinction between loss function and scoring rule even clearer, we will stick with the notion in \cref{def:level_1_scoring_rule}.

Structural properties imposed on a scoring rule allow to assess the goodness-of-fit between the distributions by means of the scoring rule. 
%
\begin{definition} \label{definition_level_1_proper}
	A (first-order) scoring rule $S_1$ is called 
	\begin{itemize}	
%
		\item \emph{regular} w.r.t.\ the class $\mathbb{P}_1(\cY)$ if $S_1(\hat p, p) \in \R$ for  all $\hat p, p \in \mathbb{P}_1(\cY)$ except possibly that  $S_1(\hat p, p)= \infty$ if $\hat p \neq p.$
		\item 	\emph{proper} w.r.t.\ the class $\mathbb{P}_1(\cY)$ if
		\begin{align} \label{def_level_1_proper}
			S_1(\hat p, p) \geq S_1(p, p) \quad \mbox{for all $\hat p, p \in \mathbb{P}_1(\cY).$}	
		\end{align}
		\item 	\emph{strictly proper} w.r.t.\ the class $\mathbb{P}_1(\cY)$ if it is proper and
		\begin{align} \label{def_level_1_strictly_proper}
			S_1(\hat p, p) > S_1(p, p) \quad \mbox{for all $\hat p \neq p.$}	
		\end{align}
	\end{itemize}
\end{definition}
Regular scoring rules assign finite scores, except that a prediction might receive an infinite score, e.g., if an event claimed to be impossible is realized.
For proper scoring rules predicting the target distribution gives the best expectation, while strictly proper scoring rules ensure that no other prediction can achieve this value.
%
%
From an uncertainty awareness perspective, the remark by \citet{gneiting2007strictly} in this regard is enlightening: ``If $S$ is proper, then the forecaster who wishes to maximize the expected score is encouraged to be honest and to volunteer his or her true beliefs.''\footnote{Note that \citet{gneiting2007strictly} consider the scenario of maximizing the score instead of minimising it as we do in this paper, which is more in line with the standard approach in machine learning.}  
%
%
Or, similarly, the one by \citet{ovcharov2018proper} regarding (strictly) proper scoring rules: ``By being maximized in expectation at the true prediction, they incentivize a forecaster to truthfully report his private information.''\footnote{Note that \citet{ovcharov2018proper} also considers maximizing the score instead of minimising it as we do in this paper.} 

\section{Scoring Rules for Second-order Losses} \label{section:level-2-scoring-rules}

Inspired by the uncertainty awareness perspective of first-order scoring rules, we ask whether one can define a similar scoring rule for second-order losses $L_2: \mathbb{P}_2(\cY) \times \cY \to \overline{\R}.$ 
Apparently, such a scoring rule needs to be a mapping from $\mathbb{P}_2(\cY) \times  \mathbb{P}_2(\cY)$ to  $\overline{\R}$ to maintain the same spirit.
\begin{definition} \label{def_score_rule_level_2}
	A (second-order) scoring rule $	S_2:   \mathbb{P}_2(\cY) \times  \mathbb{P}_2(\cY)  \to \overline{\R}$ based on the (second-order) loss $L_2: \mathbb{P}_2(\cY) \times \cY \to \overline{\R},$ such that $L_2(Q,\cdot)$ is $\mathbb{P}_2(\cY)$-quasi-integrable for all $Q \in \mathbb{P}_2(\cY)$,  is given for all $\hat Q,Q \in \mathbb{P}_2(\cY)$ by 
	\begin{align} \label{def_level_2_score}
		S_2(\hat Q, Q) = \E_{p \sim Q} \big[\E_{Y \sim p}[ L_2(\hat Q, Y) ]\big].
	\end{align}  
\end{definition}
Compared to \eqref{def_level_1_score} the definition in \eqref{def_level_2_score} involves an additional expectation w.r.t.\ the second-order distribution of the second component (i.e., $Q$).
Again, the second component represents a target (second-order) distribution, while the first component  (i.e., $\hat Q$) represents a predicted (second-order) distribution for a given instance $\vec{x}$, e.g.\ $\hat{Q}(\vec{x}) = \hat{H}(\vec{x})$.
However, unlike for first-order distributions, there is nothing like a ground-truth second-order distribution.
Nevertheless, the above definition of a second-order scoring rule allows a similar close connection to the second-order risk as for first-order scoring rules:
Suppose one would know apriori that the underlying conditional distributions, considered as random functions, are distributed according to a known second-order distribution $Q_{\vec{x}}$ (varying with the instances $\vec{x}$). 
Then, $S_2$ as in \cref{def_score_rule_level_2} relates to the risk induced by $L_2$ (see \eqref{eq:risk_level_2}) as follows:
\begin{align} \label{relation_level_2_score_risk}
	\begin{split}
		\int_\cX S_2&(\hat H(\vec{x}),Q_{\vec{x}}) \,  \mathrm{d}p_{X}(\vec{x}) 
		= \int_\cX \int_{\mathbb{P}_1(\cY)} R_2(\hat H(\vec{x}),p) \, \mathrm{d}Q_{\vec{x}}(p) \mathrm{d}p_{X}(\vec{x}), 
	\end{split}
\end{align}
where $R_2(\hat H(\vec{x}),p) = \int_{\cY} L_2(\hat H(\vec{x}),y) \, \mathrm{d}p(y)$ is the conditional risk of $\hat H$ if $p \in \mathbb{P}_1(\cY)$ is the ground-truth conditional distribution.
Thus, a risk minimising learner is automatically encouraged to minimise the second-order scoring rule in case the target second-order distribution $Q$ is known.

This connection is perhaps even more clarified in the case of learning without an instance space\footnote{Equivalently, we may assume an instance space $\cX = \{x_0\}$ consisting of only a single instance, which is observed over and over again (and can therefore be ignored, as it does not carry any
information).}, where the latter equation \eqref{relation_level_2_score_risk} boils down to
$$	S_2(\hat H,Q) = \int_{\mathbb{P}_1(\cY)} R_2(\hat H,p) \, \mathrm{d}Q(p).	$$
Akin to the first-order case (see \cref{definition_level_1_proper}) we can specify structural properties of a second-order scoring rule.
\begin{definition}
	A (second-order) scoring rule $S_2$ is called 
	\begin{itemize}
		%
		\item \emph{regular} w.r.t.\ the class $\mathbb{P}_2(\cY)$ if $S_2(\hat Q, Q) \in \R$ for  any $\hat Q, Q \in \mathbb{P}_2(\cY)$ except possibly that  $S_2(\hat Q, Q)= \infty$ if $\hat Q \neq Q.$
		\item 	\emph{proper} w.r.t.\ the class $\mathbb{P}_2(\cY)$ if
		\begin{align} \label{def_level_2_proper}
			S_2(\hat Q, Q) \geq S_2(Q, Q) \quad \mbox{for all $\hat Q, Q \in \mathbb{P}_2(\cY).$}	
		\end{align}
		\item 	\emph{strictly proper} w.r.t.\ the class $\mathbb{P}_2(\cY)$ if it is proper and
		\begin{align} \label{def_level_2_strictly_proper}
			S_2(\hat Q, Q) > S_2(Q, Q) \quad \mbox{for all $\hat Q \neq Q.$}	
		\end{align}
	\end{itemize}
\end{definition}
Given the similarity of \eqref{def_level_2_proper} (or \eqref{def_level_2_strictly_proper}) to \eqref{def_level_1_proper} (or \eqref{def_level_1_strictly_proper}) as well as the similar relationship of second-order proper scoring rules to the second-order risk as in the case of first-order, we can transfer the remarks from above for proper first-order scoring rules regarding uncertainty quantification to the second-order. 
That is, if $S_2$ is proper, then the leaner that wishes to minimise the expected score is encouraged to be honest and to volunteer its true beliefs (represented by the second component).
In other words, if the second-order distribution $Q$ is the (subjective) belief of the learner, then the best score is only obtained by using $\hat Q = Q$ as a prediction, i.e., sticking to its own belief.

In the following we derive a characterization of (strictly) proper second-order scoring rules similar to those known for (strictly) proper first-order scoring rules (see Theorem 1 in \cite{gneiting2007strictly}).
To this end, we need the definition of a concave functional on $\mathbb{P}_2(\cY)$ and its induced supertangent (or supergradient).
\begin{definition} \label{def:concavity_supertangent}
	(i) A function $G: \mathbb{P}_2(\cY) \to \R$ is \emph{concave} if for all $\lambda \in[0,1], Q,\widetilde{Q} \in \mathbb{P}_2(\cY)$ it holds that
	$$	G( \lambda Q  + (1-\lambda) \widetilde{Q}  )	\geq \lambda G(  Q)   + (1-\lambda) G(\widetilde{Q}  ).		$$
	It is \emph{strictly concave} if the latter holds with equality only in the case where $Q = \widetilde{Q}.$

	(ii) A function $G^*(Q,\cdot): \cY \to \overline{\R}$ is a supertangent of  $G$ at $\widetilde{Q} \in \mathbb{P}_2(\cY)$ if it is integrable w.r.t.\ $\widetilde{Q},$ quasi-integrable w.r.t.\ to all $Q\in \mathbb{P}_2(\cY)$ and for all $Q \in \mathbb{P}_2(\cY)$ it holds that
	\begin{align} \label{supertangent_property}
	\begin{split}
		G(Q) 
		&\leq G(\widetilde{Q}) 
		+ \int_{\mathbb{P}_1(\cY)} \int_{\cY} G^*(\widetilde{Q},y) \, \mathrm{d}p(y)\, \mathrm{d} (Q - \widetilde{Q})(p).
	\end{split}
	\end{align}
	In case the inequality in \eqref{supertangent_property} is strict for $Q \neq \widetilde{Q},$ then $G^*$ is called a strict supertangent of $G.$ 
\end{definition}
Equipped with this we can show the following characterization of (strictly) proper second-order scoring rules.
\begin{theorem} \label{theorem_level_2_scoring_rule_characterization}
	A scoring rule $S_2$  based on the (second-order) loss $L_2: \mathbb{P}_2(\cY) \times \cY \to \overline{\R}$  is (strictly) proper iff there exists a (strictly) concave function  $G_2: \mathbb{P}_2(\cY) \to \R$ such that
	\begin{align} \label{characterization_proper_level_2_scoring_rule}
		\begin{split}
			L_2(Q,y) 
			&= G_2(Q) + G_2^*(Q,y) 
			 - \int_{\mathbb{P}_1(\cY)} \int_{\cY} G_2^*(Q,y) \, \mathrm{d}p(y)\, \mathrm{d} Q(p) 
		\end{split}
	\end{align}
	for all $Q\in \mathbb{P}_2(\cY)$ and $y\in \cY,$ where $G_2^*(Q,\cdot): \cY \to \overline{\R}$ is a supertangent of $G$ at $Q.$
\end{theorem}
This theorem states in essence that a second-order scoring rule $S_2$ induced by a second-order loss $L_2$ is (strictly) proper if and only if $G_2(\cdot)=S_2(\cdot,\cdot)$ is (strictly) concave and $L_2(Q,\cdot)$ is a (strict) supertangent of $G_2$ at $Q$ for all $Q \in \mathbb{P}_2(\cY).$

Another characterization of (strictly) proper scoring rules can be derived by means of (strictly) order sensitive functions \citep{nau1985should,ovcharov2018proper}.
\begin{definition} \label{def:order-sensitive}
	A function $S: \mathbb{P}_2(\cY) \times \mathbb{P}_2(\cY) \to \overline{\R}$ is (strictly) order sensitive if the function
	\begin{align} \label{order-sensitive-function}
		\begin{split}
			f:&[0,1] \to \overline{\R} \\
			&\lambda \mapsto S( (1-\lambda)Q' + \lambda Q, Q  )
		\end{split}
	\end{align}
	is (strictly) monotonically decreasing for all $Q,Q' \in \mathbb{P}_2(\cY).$
\end{definition}
If $S$ is a (second-order) scoring rule, this property states that the score increases steadily as one moves away from the target distribution.
Unsurprisingly, there is a close connection between (strict) propriety and (strict) order sensitivity of a scoring rule, as shown in the following theorem.
\begin{theorem} \label{theorem_level_2_scoring_rule_characterization_2}
	A scoring rule $S_2$  based on the (second-order) loss $L_2: \mathbb{P}_2(\cY) \times \cY \to \overline{\R}$  is (strictly) proper iff $S_2$ is (strictly) order sensitive.
\end{theorem}
Finally, the following property of (strictly) proper second-order scoring rules is useful as it allows to normalize the scores if necessary.
\begin{lemma} \label{lemma_scaling_l2_losses}
	If $L_2: \mathbb{P}_2(\cY) \times \cY \to \overline{\R}$  induces a (strictly) proper scoring rule $S_2,$ then $\tilde L_2(Q,y) = c L_2(Q,y) + g(y)$ for any constant $c>0$ and $\mathbb{P}_2(\cY)$-integrable function $g:\cY \to \R$ induces a (strictly) proper scoring rule $\tilde S_2.$
\end{lemma}

\section{Non-existence of Proper Second-order Scoring Rules} \label{section:no-proper-level-2-scoring-rules}

In this section, we use the characterizations of (strictly) proper second-order scoring rules derived above to show negative results regarding the existence of a reasonable second-order loss $L_2$ such that the induced second-order scoring rule is (strictly) proper.
Here, reasonable refers to loss functions that are desired from an optimization perspective, i.e., almost continuous, as well as an uncertainty penalization perspective, which we shall discuss in more detail after each theoretical result.

\subsection{Classification}
We start with the classification setting, i.e., $\cY=\{y_1,\ldots,y_K\}$ for some $K\in \mathbb{N}_{\geq 2}.$ 
Note that any probability distribution $p$ on $\cY$ is characterized by a probability mass function, which we shall also denote by $p.$ 
\begin{theorem} \label{theorem_no_proper_level_2_scoring_rule_classification}
	%
	%
	There exists no loss function $L_2:  \mathbb{P}_2(\cY) \times \cY \to \R$ such that the induced second-order scoring rule $S_2$ is proper if either of the following holds for $L_2:$
	\begin{enumerate}
		[label=(\roman*)]
%
%
		\item  $L_2(\cdot,y'')$ is almost continuous for all $y''\in \cY$ and for all $y\in \cY,$ $Q,\overline{Q}\in \mathbb{P}_2(\cY),$  it holds that
		\begin{align} \label{no_proper_level_2_scoring_rule_condition_2}
			\begin{split}
				&L_2(Q, y)    < L_2(\overline{Q}, y)
			\end{split}
		\end{align}
		iff  $ \E_{p\sim Q}[p(y)] > \E_{p\sim \overline{Q}}[p(y)].$
		\item there exist $y\in \cY$  and $Q,\overline{Q} \in \mathbb{P}_2(\cY)$ such that
		\begin{align} \label{no_proper_level_2_scoring_rule_condition_3}
			&L_2(\overline{Q},y) < L_2(Q, y) \, \mbox{\&} \, \E_{p\sim \overline{Q}}[p(y)] < \E_{p\sim Q}[p(y)] 
		\end{align}
        and 
        \begin{align}\label{no_proper_level_2_scoring_rule_condition_4}
            \begin{split}
                   &\sum_{y_k \neq y} \big(L_2(\overline{Q},y_k) - L_2(\overline{Q},y) \big) 
                   \leq \sum_{y_k \neq y}  \big(L_2(Q, y_k) - L_2(Q, y)\big).
            \end{split}     
        \end{align}
%
		%
	\end{enumerate}
%
%
	%
	\normalsize
	%
\end{theorem}
\begin{proof}
	\emph{Case (i).} Assume that $L_2$ satisfies (i). 
	It is sufficient to show the assertion for the binary classification case with $K=2,$ by considering the subset of $\mathbb{P}_2(\cY)$ which has only support on first-order distributions which in turn have only support on two fixed classes.
	
	For ease of notation, let us use the encoding $y_1 = 0$ and $y_2 = 1,$ so that $\cY=\{0,1\}.$  		
	In light of \cref{theorem_level_2_scoring_rule_characterization_2} the function in \eqref{order-sensitive-function} for $S=S_2$ needs to be (strictly) monotonically decreasing if $S_2$ is (strictly) proper. 
	Thus, for all $\lambda \in[0,1], Q, Q' \in \mathbb{P}_2(\cY)$ it must hold that
	\begin{align*}
		S_2( \widetilde{Q} , Q )	\geq  S_2(  Q  , Q  ),
	\end{align*}
	where we abbreviated $\widetilde{Q} = \lambda Q  + (1-\lambda) Q'.$ 
	This is equivalent to
	\begin{align*}
		&\int_{\mathbb{P}_1(\cY)}  L_2(\widetilde{Q} , 0) p(0) +  L_2(\widetilde{Q} , 1) p(1)  \,  \mathrm{d}  Q(p) 
		\geq \int_{\mathbb{P}_1(\cY)}  L_2(Q , 0) p(0) +  L_2(Q , 1) p(1)  \,  \mathrm{d}  Q(p) 
	\end{align*}
	which due to $p(0) = 1 - p(1)$ for all $p \in \mathbb{P}_1(\cY)$ can be further rewritten to
	\begin{align} \label{necessary_property}
		\begin{split}
			&\int\limits_{\mathbb{P}_1(\cY)}  L_2(\widetilde{Q} , 0) +  \Big(  L_2(\widetilde{Q} , 1) -L_2(\widetilde{Q} , 0)  \Big) p(1)  \,  \mathrm{d}  Q(p) 
			\geq \int\limits_{\mathbb{P}_1(\cY)}  L_2(Q , 0) +   \Big( L_2(Q , 1) -  L_2(Q , 0) \Big) p(1)  \,  \mathrm{d}  Q(p).
		\end{split}
	\end{align}
%
%
%
%
	%
	%
	%
        \clearpage
	Choose $Q,Q'\in \mathbb{P}_2(\cY)$ and $\lambda$ such that 
        \begin{itemize}
            %
            \item $\widetilde{Q} = \delta_{\delta_0},$ i.e., the second-order Dirac measure, which puts all its mass on the first-order Dirac measure at 0, 
            \item $0 <	\E_{p\sim Q}[p(1)],$ which can be achieved as soon as $Q$ assigns mass to first-order distributions, which are assigning positive mass to class $1,$
            \item $\E_{p\sim Q}[p(1)] < \Big( 1+ \frac{ L_2(\widetilde{Q},1) - L_2(Q,1)  }{ L_2(Q,0) - L_2(\widetilde{Q},0)}  \Big)^{-1} < 1,$ which can be achieved, since  $L_2(\cdot,y'')$ is by assumption almost continuous for all $y''\in \cY,$ and \eqref{no_proper_level_2_scoring_rule_condition_2} together with the previous two properties of $Q$ and $\widetilde{Q}$ implies
	$$ \min\Big(L_2(\widetilde{Q},1) - L_2(Q,1)  , L_2(Q,0) - L_2(\widetilde{Q},0) \Big)> 0.$$ 
        \end{itemize}  
	%
		%
		%
	%
	%
	%
	%
	Then, \eqref{necessary_property} is violated, since
	\begin{align*} 
		%
			%
			%
			&\E_{p\sim Q}[p(1)] < \Big( 1+ \frac{ L_2(\widetilde{Q},1) - L_2(Q,1)  }{ L_2(Q,0) - L_2(\widetilde{Q},0)}  \Big)^{-1} \\
			%
			&\Leftrightarrow  L_2(\widetilde{Q},0) +  \E_{p\sim Q}[p(1)] \Big( L_2(\widetilde{Q},1) - L_2(\widetilde{Q},0)  \Big) 
			 <  L_2(Q,0) +  \E_{p\sim Q}[p(1)] \Big( L_2(Q,1) - L_2(Q,0)  \Big) \\
			&\Leftrightarrow  \int\limits_{\mathbb{P}_1(\cY)}  L_2(\widetilde{Q} , 0) +  \Big(  L_2(\widetilde{Q} , 1) -L_2(\widetilde{Q} , 0)  \Big) p(1)  \,  \mathrm{d}  Q(p)   <  \int\limits_{\mathbb{P}_1(\cY)}  L_2(Q , 0) +   \Big( L_2(Q , 1) -  L_2(Q , 0) \Big) p(1)  \,  \mathrm{d}  Q(p).
			%
		%
	\end{align*}

	\emph{Case (ii).} Assume that $L_2$ satisfies (ii). 
        Similarly to \eqref{necessary_property} we can derive that 
        \begin{align} \label{necessary_property_case_2}
		\begin{split}
			&\int\limits_{\mathbb{P}_1(\cY)}  L_2(\widetilde{Q} , y)  +  \sum_{y_k \neq y} \big(L_2(\widetilde{Q},y_k) - L_2(\widetilde{Q},y) \big) p(y_k)  \,  \mathrm{d}  Q(p) \\
			&\geq \int\limits_{\mathbb{P}_1(\cY)}  L_2(Q , y)  +  \sum_{y_k \neq y}  \big(L_2(Q, y_k) - L_2(Q, y)\big) p(y_k)  \,  \mathrm{d}  Q(p).
		\end{split}
	\end{align}
    must hold if $S_2$ is proper, since $p(y) = 1 - \sum_{y_k \neq y}   p(y_k).$
    %
%
%
%
    However, if $\lambda \in [0,1], Q' \in \mathbb{P}_2(\cY)$ is such that $\widetilde{Q} = \overline{Q}$ (which is possible due to convexity of $\mathbb{P}_2(\cY)$), then \eqref{necessary_property_case_2} is violated, due to \eqref{no_proper_level_2_scoring_rule_condition_3} and \eqref{no_proper_level_2_scoring_rule_condition_4}.
%
 %
\end{proof}
Note that the two cases are not entirely exhaustive, as condition \eqref{no_proper_level_2_scoring_rule_condition_4} is required additionally to condition \eqref{no_proper_level_2_scoring_rule_condition_3}, but requiring only the latter condition would correspond to the complementary condition of \eqref{no_proper_level_2_scoring_rule_condition_2}.
Nevertheless,  condition \eqref{no_proper_level_2_scoring_rule_condition_4} essentially requires the loss function to be cost-insensitive regarding any two classes, which is in absence of additional a priori knowledge on the data set not too restrictive.
Condition \eqref{no_proper_level_2_scoring_rule_condition_3} is fulfilled if the second-order loss function $L_2$ penalizes second-order point predictions (i.e., $\delta_{p}$) more drastically as for instance (second-order) predictions which are slightly deviating from  point predictions.
This is the case for loss functions with a regularization term that introduces a bias towards the second-order uniform distribution \citep{sens_ed18,char_ue20,tsiligkaridis2021information}.
As a consequence, the (loss-minimising) learner has a tendency to predict more flat distributions.
On the other hand, the condition in \eqref{no_proper_level_2_scoring_rule_condition_2} enforces second-order point predictions to be more concentrated for the correct class, and avoid concentration for incorrect classes, so that the learner has a tendency to predict more peaked distributions.

The results complement those of \citet{bengs2022pitfalls} for the empirical risk minimiser for the existing second-order losses (see \eqref{def:bayes_losses}), since losses fulfilling \eqref{no_proper_level_2_scoring_rule_condition_3} and \eqref{no_proper_level_2_scoring_rule_condition_4} are a generalization of the Bayesian losses with a too large regularisation parameter, while losses fulfilling \eqref{no_proper_level_2_scoring_rule_condition_2} generalize the Bayesian losses with a too low regularisation parameter. 

\subsection{Regression}
Next, we consider the case of regression, i.e., $\cY=\R.$
%
%
\begin{theorem} \label{theorem_no_proper_level_2_scoring_rule_regression}
	There exists no loss function $L_2:  \mathbb{P}_2(\cY) \times \cY \to \R$ such that the induced second-order scoring rule $S_2$ is proper if either of the following holds for $L_2:$
	\begin{enumerate}
		[label=(\roman*)]
		\item  $L_2(\cdot,y'')$ is almost continuous for all $y''\in \cY$ and for all $y\in \cY,$ $Q,\overline{Q}\in \mathbb{P}_2(\cY),$  it holds that		
		\begin{align} \label{no_proper_level_2_scoring_rule_regression_condition_1}
			\begin{split}
				&L_2(Q, y)    < L_2(\overline{Q}, y)
			\end{split}
		\end{align}
		iff  $ | \E_{p\sim Q}[\E(p) ] -y | < | \E_{p\sim \overline{Q}}[\E(p) ] - y |.$
		\item there exist $ \mu \in \cY,$ a first-order distribution $\tilde p \in \mathbb{P}_1(\cY)$ with mean $\mu,$  a second-order distribution $\overline{Q} \in \mathbb{P}_2(\cY),$ and some $\delta>0$ such that for almost all \footnote{A condition holds for \emph{almost all} $x$ in some set $X,$ if the subset on which the condition does not hold has probability mass 0.} $y \in (\mu-\delta,\mu+\delta)$ it holds that
		\begin{align} \label{no_proper_level_2_scoring_rule_regression_condition_2}
			&\quad L_2(\overline{Q},y) < L_2(\delta_{\tilde  p}, y) \\
			%
			\begin{split}  \label{no_proper_level_2_scoring_rule_regression_condition_3}
				&\mbox{and} \quad \int_{ \big((\mu-\delta,\mu+\delta)\big)^\complement  } L_2(\overline{Q},y)\, \mathrm{d}\tilde p(y) < \int_{ \big((\mu-\delta,\mu+\delta)\big)^\complement  } L_2(\delta_{p},y)\, \mathrm{d}\tilde p(y),
			\end{split}
		\end{align}	
            where $  \big((\mu-\delta,\mu+\delta)\big)^\complement  = \cY \setminus (\mu-\delta,\mu+\delta).$
		%
%
	\end{enumerate}
\end{theorem}
\begin{proof}
	\emph{Case (i).}
	Let $p_l,p_r \in \mathbb{P}_1(\cY)$ be first-order distributions and  $\mu^* \in \R$ such that 
	\begin{itemize}
%
		\item $p_l$ has only support on $(-\infty,\mu^*)$ and $p_r$ has only support on $(\mu^*,\infty),$
		\item $p_l$ has expected value $\mu_l < \mu^*$ and $p_r$ has an expected value $\mu_r > \mu^*,$
		\item it holds  for all $\lambda \in(0,1)$ that
		\begin{align*}
			\E_{Y \sim p_l} [  L_2(Q_\lambda,Y )  - L_2(\delta_{p_l},Y)]>0,
		\end{align*}
		where $Q_\lambda = (1-\lambda) \delta_{p_l} + \lambda \delta_{p_r}.$
	\end{itemize}
	This is possible, since 
	\begin{align*}
		\E_{Y \sim p_l} [  L_2( Q_\lambda,Y )  - L_2(\delta_{p_l},Y)] 
		&= \int_{(-\infty,\mu_l]} L_2( Q_\lambda,y )  - L_2(\delta_{p_l},y) \, \mathrm{d}p_l(y) \\
		 &\quad + \int_{[\mu_l,\mu^*)} L_2( Q_\lambda,y )  - L_2(\delta_{p_l},y) \, \mathrm{d}p_l(y)
	\end{align*}
	and the first term is always positive by \eqref{no_proper_level_2_scoring_rule_regression_condition_1}, as $  \E_{p\sim \delta_{p_l}}[\E(p) ] = \mu_l$ and $ \E_{p\sim Q_\lambda}[\E(p) ] = (1-\lambda)\mu_l + \lambda \mu_r > \mu_l.$ 
	Thus, by suitable choice of $\mu^*,$ $p_l$ and $p_r$ (as well as $\mu_l$ and $\mu_r$), the second term can be designed such that it is smaller than the first in absolute terms, since $L_2(\cdot,y'')$ is by assumption almost continuous for all $y''\in \cY.$

	Let $\widetilde{Q} = \delta_{p_l},$ and choose $Q \in \mathbb{P}_2(\cY)$ such that $Q = (1- \lambda^*) \delta_{p_l} + \lambda^* \delta_{p_r},$ where  
	\begin{align} \label{def_lambda_choice}
		0 < \lambda^* < \left( 1 + \frac{\E_{Y\sim p_r}( L_2(\widetilde{Q},Y) - L_2(Q,Y) )}{\E_{Y\sim p_l}( L_2(Q,Y) - L_2(\widetilde{Q},Y) )}		\right)^{-1}< 1.
	\end{align}
	This choice of $\lambda^*$ can be achieved as $\E_{Y\sim p_l}( L_2(Q,Y) - L_2(\widetilde{Q},Y) )>0$ by choice of $\mu^*,$ $p_l$ and $p_r,$ and \eqref{no_proper_level_2_scoring_rule_regression_condition_1} implies that $\E_{Y\sim p_r}( L_2(\widetilde{Q},Y) - L_2(Q,Y) )>0,$ since 	for all $y$ in the support of $p_r$ it holds that $L_2(\widetilde{Q},y) - L_2(Q,y)>0.$  
	Thus, 
	\begin{align*}
		&S_2(\widetilde{Q},Q)\\
		&= \int_{\mathbb{P}_1(\cY)} \int_{\cY}    L_2(\widetilde{Q},y) \,  \mathrm{d}p(y)\, \mathrm{d}  Q(p) \\
		&= (1- \lambda^*) \int\limits_{\cY}    L_2(\widetilde{Q},y) \,  \mathrm{d}p_l(y) 
		+ \lambda^* \int\limits_{\cY}    L_2(\widetilde{Q},y) \,  \mathrm{d}p_r(y) \\
		&= (1- \lambda^*) \E_{Y\sim p_l}( L_2(\widetilde{Q},Y)) + \lambda^* \E_{Y\sim p_r}( L_2(\widetilde{Q},Y)) \\
		&< (1- \lambda^*) \E_{Y\sim p_l}( L_2(Q,Y)) + \lambda^* \E_{Y\sim p_r}( L_2(Q,Y)) \\
		&= S_2(Q,Q),
	\end{align*}
	where the inequality is due to \eqref{def_lambda_choice}.

	\emph{Case (ii).}
	Let us abbreviate $(\mu-\delta,\mu+\delta) $ by $(\mu\pm\delta) .$
	Choose $Q$ to be $\delta_{\tilde p},$ then \eqref{no_proper_level_2_scoring_rule_regression_condition_2} and \eqref{no_proper_level_2_scoring_rule_regression_condition_3} imply that
	\begin{align*}
		S_2&(\overline{Q}, Q) \\
		&= \int_{\mathbb{P}_1(\cY)} \int_{\cY}    L_2(\overline{Q},y) \,  \mathrm{d}p(y)\, \mathrm{d}  Q(p) \\
		&=  \int_{\cY}    L_2(\overline{Q},y) \,  \mathrm{d} \tilde p(y) \\
		&= \int_{ (\mu\pm\delta)  }   L_2(\overline{Q},y) \,  \mathrm{d} \tilde p(y)  +  \int_{ \big( \mu \pm \delta \big)^\complement  }   L_2(\overline{Q},y) \,  \mathrm{d} \tilde p(y) \\
		&< \int_{ (\mu \pm \delta)  }  L_2(\delta_{\tilde  p}, y) \,  \mathrm{d} \tilde p(y)  +  \int_{ \big(\mu \pm \delta\big)^\complement  }  L_2(\delta_{\tilde  p}, y) \,  \mathrm{d} \tilde p(y) \\
		&= \int_{ (\mu \pm \delta)  }  L_2(Q, y) \,  \mathrm{d} \tilde p(y)  +  \int_{ \big(\mu \pm \delta\big)^\complement  }  L_2(Q, y) \,  \mathrm{d} \tilde p(y) \\
		&= \int_{\cY}    L_2(Q,y) \,  \mathrm{d} \tilde p(y) \\
		&= \int_{\mathbb{P}_1(\cY)} \int_{\cY}    L_2(Q,y) \,  \mathrm{d}p(y)\, \mathrm{d}  Q(p) = S_2(Q,Q).
	\end{align*}
	Thus, $S_2$ is not proper.
\end{proof}
The two cases in \cref{theorem_no_proper_level_2_scoring_rule_regression} are quite similar to the the ones in \cref{theorem_no_proper_level_2_scoring_rule_classification}: the first corresponds to second-order losses that incentive the learner to predict more flat distributions, while the second incentives predictions of peaked distributions.
%
%
The second-order loss function for the regression case suggested by  \citet{amini2020deep} (see \cref{example:Regression}) fulfills the conditions in the second case of \cref{theorem_no_proper_level_2_scoring_rule_regression} as the following proposition shows.

\begin{proposition}
	The deep evidential regression loss function $L_2^{\texttt{DER}}: \mathbb{P}_2(\cM) \times \cY \to \R$ in \eqref{def:DER_loss} fulfills the conditions in \cref{theorem_no_proper_level_2_scoring_rule_regression} (ii).
\end{proposition}
\begin{proof}
	Let $\mu \in \R$ be arbitrary but fixed and $\sigma >0$ be some small value.
	Note that an NIG distribution with parameters $m_1=\mu,$ $m_2=\infty$ and $m_3,m_4 $ such that $ m_4/(m_3-1) = \sigma^2$ and $m_3$ sufficiently large corresponds to $\delta_{\tilde p}$ with $\E(\tilde p) = \mu$ and $\mathbb{V}(\tilde p) = \sigma^2.$
	However, by using for $ \overline{Q}$ an NIG distribution with $\tilde m_1 = m_1,$ $\tilde m_2 = 1,$ $\tilde m_3 = 1$ and $\tilde m_4 = m_4,$  the deep evidential regression loss of $\delta_{\tilde  p}$ is larger for all $y$ except for $y=\mu$ than the deep evidential regression loss of $\overline{Q}.$
\end{proof}

%

%


\section{Conclusion}
Our results confirm concerns raised by recent work regarding the conceptual meaningfulness of direct epistemic uncertainty quantification through empirical risk minimisation of second-order distributions.
More precisely, unlike for the case of empirical risk minimisation of strictly proper first-order loss functions to report first-order (aleatoric) uncertainty in a faithful manner, there seems to be no strictly proper second-order loss function counterpart to report second-order (epistemic) uncertainty in a faithful manner.

The most likely explanation is the discrepancy between the orders that these second-order loss functions exhibit: a second-order prediction $H(\vec{x})$ is evaluated in light of a zero-order observation $y$, skipping the intervening first-order. 
Thus, to make the loss function meaningful, one would rather need observations of realizations of the first-order distribution, i.e., a sample in the form of probabilities and asses the second-order prediction  in light of these first-order observations. 
However, such data cannot exist even in principle, because the ground-truth conditional distribution is supposedly constant. 

This suggests that (probabilistic) learning on the epistemic level cannot be frequentist in nature, unlike learning about the ground-truth conditional distribution on the first-order (aleatoric level). 
Instead, it appears that learning on the second-order (epistemic level) is necessarily
Bayesian and requires a prior, which then of course has an influence on the degree of (epistemic) uncertainty.

\section*{Acknowledgments and Disclosure of Funding}

Willem Wageman received funding from the Flemish Government under the ``Onderzoeksprogramma Artificiel\"e Intelligentie (AI) Vlaanderen'' Programme.

\bibliographystyle{icml2022}
\bibliography{ref}
\clearpage

\newpage
\appendix
\onecolumn

\section{List of Symbols}

The following table contains a list of symbols that are frequently used in the main paper as well as in the following supplementary material. \\ \medskip
\small

\begin{table}[ht!]
	\begin{centering}
		\resizebox{0.99\textwidth}{!}{
			\begin{tabular}{l|l}
				\hline
				\multicolumn{2}{c}{\textbf{General Learning Setting}} \\
				\hline 
				$\cX$ & instance space \\
				$\cY$ & label space, either $\{y_1 , \ldots , y_K\}$ for some $K \in \mathbb N_{\geq 2}$ for classification or $\cY = \R$ for regression \\
				$\cD$ & training data $\big\{ \big(\vec{x}^{(n)} , y^{(n)} \big) \big\}_{n=1}^N \subset \cX \times \cY \, $ \\
				$p^*$ & data generating probability \\
				$\prob^*( \cdot \given \vec{x})$ &  conditional distribution or density on $\cY$, i.e., $\prob^*( y \given \vec{x})$ probability to observe $y$ given $\vec{x}$ in classification \\
				& density of $y$ given $\vec{x}$  in regression \\
				$\mathbb{P} ( \cY),\mathbb{P}_1 ( \cY)$  & the set of probability distributions on $\cY$ \\
				$\mathbb{P}_1 ( \Theta)$  &  a parameterized subset of $\mathbb{P}_1 ( \cY)$ with $\Theta$ being the parameter space  \\
				%
				%
				%
				\hline
				\multicolumn{2}{c}{\textbf{First-order Learning Setting}} \\
				\hline
				$\mathcal{H}_1$ & (first-order) hypothesis space consisting of hypothesis $h :\cX \fromto \mathbb{P}_1 ( \cY)$ \\
				$L_1$ & loss function for first-order hypothesis, i.e., $L_1: \mathbb{P}_1 ( \cY) \times \cY \fromto \mathbb{R}$ \\
				$R_{1,emp}(\cdot)$ & empirical risk of a first-order hypothesis (cf.\ \eqref{eq:erisk}) \\
				$R_1(\cdot)$ & risk or expected loss of a first-order hypothesis (cf.\ \eqref{eq:risk}) \\
				$\hath$ & empirical risk minimiser, i.e., $\hath = \argmin_{h \in \cH_1} R_{1,emp}(h)$ \\
				$h^*$ & true risk minimiser or Bayes predictor, i.e., $h^* = \argmin_{h \in \cH_1} R_1(h)$\\
				$S_1(\cdot,\cdot)$ & first-order scoring rule induced by some first-order loss function $L_1$, \\
				& i.e., $	S_1:   \mathbb{P}_1(\cY) \times  \mathbb{P}_1(\cY)  \to \overline{\R}$ with $S_1(\hat p, p) = \E_{Y \sim p}[ L_1(\hat p, Y) ]$ (see \eqref{def_level_1_score}) \\
				%
				\hline
				\multicolumn{2}{c}{\textbf{Second-order Learning Setting}} \\
				\hline
				$\mathbb{P}_2 ( \cY)$ & the set of distributions on $\mathbb{P}_1 ( \cY)$ (the set of second-order distributions) \\
				$\mathbb{P}_2 ( \cM)$ & a parameterized subset of $\mathbb{P}_2 ( \cY)$ with $\cM$ being the parameter space  \\
				$\mathcal{H}_2$ & (second-order) hypothesis space consisting of hypothesis $H :\cX \fromto \mathbb{P}_2 ( \cY)$ \\
				%
				%
				%
				$Q,Q',\bar{Q},\widetilde{Q}$ & probability distributions on $\mathbb{P}_1 ( \cY)$ i.e., elements of $\mathbb{P}_2 ( \cY)$ \\
				$Q_{0}$ & uniform distribution on $\mathbb{P}_1 ( \cY)$ (an element of $\mathbb{P}_2 ( \cY)$) \\
				$L_2$ & loss function for second-order hypothesis, i.e., $L_2: \mathbb{P}_2 ( \cY) \times \cY \fromto \mathbb{R}$ \\
				$L^{\texttt{Bay}}_2$ & Bayesian loss functions for classification setting (see \eqref{def:bayes_losses}) \\
				$L^{\texttt{DER}}$ & deep evidential regression loss functions for regression setting (see \eqref{def:DER_loss}) \\
				$R_{2,emp}(\cdot)$ & empirical risk of a second-order hypothesis (cf.\ \eqref{eq:erisk_level_2}) \\
				$R_2(\cdot)$ & risk or expected loss of a second-order hypothesis (cf.\ \eqref{eq:risk_level_2}) \\
				$\hat H$ & second-order empirical risk minimiser, i.e., $\hat H = \argmin_{H \in \cH_2} R_{2,emp}(H)$ \\
				$S_2(\cdot,\cdot)$ & second-order scoring rule induced by some second-order loss function $L_2$, \\
				& i.e., $	S_2:   \mathbb{P}_2(\cY) \times  \mathbb{P}_2(\cY)  \to \overline{\R}$ with $S_2(\hat Q, Q) = \E_{p \sim Q} \big[\E_{Y \sim p}[ L_2(\hat Q, Y) ]\big]$ (see \eqref{def_level_2_score}) \\
				\hline
				\multicolumn{2}{c}{\textbf{Distributions \& Expectations}} \\
				\hline
				%
				%
				%
				$\mathrm{N}(\mu,\sigma^2)$ & Gaussian distribution with location parameter $\mu$ and scale parameter $\sigma>0$ \\
				%
				%
				$\delta_{y}$ & Dirac measure at $y \in \cY$ (i.e., $\delta_{y}$ is an element $\mathbb{P}_1(\cY)$ ) \\			
				$\delta_{p}$ & Dirac measure at $p \in  \mathbb{P}_1(\cY)$ (i.e., $\delta_{p}$ is an element $\mathbb{P}_2(\cY)$ ) \\
				%
%
				$\E(p)$ & expected value of the distribution $p\in \mathbb{P}_1(\R),$ i.e., $\E(p) = \int y \, \mathrm{d}p(y)$ \\
				$\mathbb{V}(p)$ & variance of a distribution $p\in \mathbb{P}_1(\mathbb R),$ i.e., $\mathbb{V}(p) = \E \big[ (p - \E[p])^2 \big]$ \\
				$\E_{p\sim Q}[p(y)]$ & expected probability assigned to class $y \in \cY$ (i.e., classification setting) according to $Q \in \mathbb{P}_2(\cY),$ \\
				& i.e., $ \E_{p\sim Q}[p(y)] = \int_{\mathbb{P}_1(\cY)} p(y) \mathrm{d}Q(p)  $\\
				$\E_{p\sim Q}[\E(p) ] $ & expected value of the expected distribution (for regression, i.e., $\cY = \R$) according to $Q \in \mathbb{P}_2(\cY),$ \\
				& i.e., $ \E_{p\sim Q}[\E(p) ] = \int_{\mathbb{P}_1(\cY)} \int_{\cY} y \, \mathrm{d} p(y) \mathrm{d}Q(p)  $ \\
				\hline
				\multicolumn{2}{c}{\textbf{Miscellaneous}} \\
				\hline
				%
				%
				$d_{KL}\left(\cdot, \cdot\right)$ & Kullback-Leibler divergence (on $\mathbb{P}_2(\cY) \times \mathbb{P}_2(\cY)$)\\
				$L_1^{\texttt{Brier}}$ & Brier score (see \eqref{def:Brier_score}) \\
				$L_1^{\texttt{CE}}$ & cross-entropy loss (see \eqref{def:cross_entropy_loss}) \\
				$L^{\texttt{t}}$ & negative log-likelihood of Student-t distribution (see \eqref{def:neg_log_likeli_student}) \\
				$\mathrm{PEN}$ & penalization function in deep evidential regression loss function   (see \eqref{def:neg_log_likeli_student}) \\
				%
%
		\end{tabular}}
	\end{centering}
\end{table}
%
\normalsize
\clearpage

\section{Missing Proofs of \cref{section:level-2-scoring-rules}}

\subsection{Proof of \cref{theorem_level_2_scoring_rule_characterization}}

We will use the following lemma for the proof of \cref{theorem_level_2_scoring_rule_characterization}, which we shall prove at the end of this subsection. 
\begin{lemma} \label{lemma_supertangent_concave_implication}
	If a function $G:\mathbb{P}_2(\cY) \to \R$ has a supertangent $G^*(Q,\cdot)$ at any $Q \in \mathbb{P}_2(\cY),$ then $G$ is concave.
	If the supertangent property in \eqref{supertangent_property} holds with strict inequality for all $Q \neq \widetilde{Q},$ then $G$ is strictly concave.
\end{lemma}

\begin{proof} 
	%
	%
	%
	Suppose $S_2$ (or rather $L_2$) fulfills the representation in \eqref{characterization_proper_level_2_scoring_rule}, then
	\begin{align*}
		S_2&(\hat Q, Q) \\
		&= \E_{p \sim Q} \big[\E_{Y \sim p}[ L_2(\hat Q, Y) ]\big] \\
		&= \int_{\mathbb{P}_1(\cY)} \int_{\cY} L_2(\hat Q, y)  \, \mathrm{d}p(y)\, \mathrm{d}Q(p)  \\
		&\stackrel{\eqref{characterization_proper_level_2_scoring_rule}}{=} G_2(\hat Q) + \int_{\mathbb{P}_1(\cY)} \int_{\cY} G_2^*(\hat Q,y) \, \mathrm{d}p(y)\, \mathrm{d} Q(p) \\
		&\quad -  \int_{\mathbb{P}_1(\cY)} \int_{\cY} G_2^*(\hat Q,y) \, \mathrm{d}p(y)\, \mathrm{d} \hat Q(p) \\
		&= G_2(\hat Q) + \int_{\mathbb{P}_1(\cY)} \int_{\cY} G^*(\hat Q,y) \, \mathrm{d}p(y)\, \mathrm{d} ( Q - \hat Q)(p) \\
		&\stackrel{\eqref{supertangent_property}}{\geq} G_2(Q) \\
		&= G_2(Q) + \int_{\mathbb{P}_1(\cY)} \int_{\cY} G_2^*(Q,y) \, \mathrm{d}p(y)\, \mathrm{d} Q(p) \\
		&\quad -  \int_{\mathbb{P}_1(\cY)} \int_{\cY} G_2^*(Q,y) \, \mathrm{d}p(y)\, \mathrm{d} Q(p) \\
		&\stackrel{\eqref{characterization_proper_level_2_scoring_rule}}{=} \int_{\mathbb{P}_1(\cY)} \int_{\cY} L_2(Q, y)  \, \mathrm{d}p(y)\, \mathrm{d}Q(p) \\
		&= \E_{p \sim Q} \big[\E_{Y \sim p}[ L_2(Q, Y) ]\big] \\
		&= S_2(Q, Q),
	\end{align*}
	where we used for the inequality that $G_2^*$ is a supertangent of $G$ at $\hat Q,$ i.e., that $\eqref{supertangent_property}$ holds for $Q.$
	Note that the inequality is strict if $G_2$ is strictly concave and $\hat Q \neq Q.$
	
	Conversely, suppose $S_2$ to be (strictly) proper scoring rule. 
	Define $G_2$ by $G_2(Q) = S_2(Q,Q),$ then $G_2^*(Q,y) = L_2(Q,y)$ is a supertangent of $G_2$ at any $\widetilde{Q} \in \mathbb{P}_2(\cY).$
	Indeed, let $Q \in \mathbb{P}_2(\cY),$ then
	\begin{align*}
		G_2&(Q) \\
		&= S_2(Q,Q) \\
		&= S_2(\widetilde{Q},\widetilde{Q}) -  S_2(\widetilde{Q},\widetilde{Q}) + S_2(\widetilde{Q}, Q) -  S_2(\widetilde{Q}, Q)  \\
		&\quad +  S_2(Q, Q) \\
		&\leq S_2(\widetilde{Q},\widetilde{Q}) -  S_2(\widetilde{Q},\widetilde{Q}) + S_2(\widetilde{Q}, Q) \\
		&=  S_2(\widetilde{Q}, \widetilde{Q})  + \int_{\mathbb{P}_1(\cY)} \int_{\cY} L_2(\widetilde{Q},y) \, \mathrm{d}p(y)\, \mathrm{d} (Q - \widetilde{Q})(p) \\
		&= G_2(\widetilde{Q}) + \int_{\mathbb{P}_1(\cY)} \int_{\cY} G_2^*(\widetilde{Q},y) \, \mathrm{d}p(y)\, \mathrm{d} (Q - \widetilde{Q})(p),
	\end{align*}
	where for the inequality we used that $S_2$ is proper.
	This inequality is strict for all $\widetilde{Q} \neq Q$ if $S_2$ is strictly proper.
	Thus, $G_2$ is (strictly) concave due to \cref{lemma_supertangent_concave_implication}. 
	The definitions of $G_2$ and $G_2^*$ directly imply that $L_2$ has the representation in \eqref{characterization_proper_level_2_scoring_rule}, since $G_2(Q) = \int_{\mathbb{P}_1(\cY)} \int_{\cY} G_2^*(Q,y) \, \mathrm{d}p(y)\, \mathrm{d} Q(p).$
	
\end{proof}

\begin{proof}[Proof of \cref{lemma_supertangent_concave_implication}]
	Let $Q,\widetilde{Q} \in \mathbb{P}_2(\cY)$ and $\lambda \in[0,1]$ be arbitrary but fixed.
        Abbreviate $Q_\lambda = (1-\lambda) Q +\lambda \widetilde{Q}.$ 
        Then, since $G^*$ is by assumption a supertangent of $G$ for any element of $\mathbb{P}_2(\cY),$ it holds by the supertangent property (see \eqref{supertangent_property}) that
        \begin{align*}
            &G(Q) \leq G(Q_\lambda) + \int_{\mathbb{P}_1(\cY)} \int_{\cY} G^*(Q_\lambda,y) \, \mathrm{d}p(y)\, \mathrm{d} (Q - Q_\lambda)(p), \\
            &G(\widetilde{Q} ) \leq G(Q_\lambda) + \int_{\mathbb{P}_1(\cY)} \int_{\cY} G^*(Q_\lambda,y) \, \mathrm{d}p(y)\, \mathrm{d} (\widetilde{Q}  - Q_\lambda)(p).
        \end{align*}
        If we consider the convex combination of these two inequalities and noting that $ Q - Q_\lambda = \lambda Q  -\lambda \widetilde{Q}  $ as well as $ \widetilde{Q} -Q_\lambda = (1-\lambda) \widetilde{Q} - (1-\lambda) Q,$ we obtain
        \begin{align*}
                (1-\lambda) G(Q) + \lambda G(\widetilde{Q}) 
                &\leq  (1-\lambda)  G(Q_\lambda) &&+  (1-\lambda) \int_{\mathbb{P}_1(\cY)} \int_{\cY} G^*(Q_\lambda,y) \, \mathrm{d}p(y)\, \mathrm{d} (Q - Q_\lambda)(p) \\
                &\quad + \lambda G(Q_\lambda)  &&+ \lambda \int_{\mathbb{P}_1(\cY)} \int_{\cY} G^*(Q_\lambda,y) \, \mathrm{d}p(y)\, \mathrm{d} (\widetilde{Q}  - Q_\lambda)(p) \\
                &=  (1-\lambda)  G(Q_\lambda) &&+  (1-\lambda)\lambda \int_{\mathbb{P}_1(\cY)} \int_{\cY} G^*(Q_\lambda,y) \, \mathrm{d}p(y)\, \mathrm{d} (Q - \widetilde{Q})(p) \\
                &\quad + \lambda G(Q_\lambda)  &&+ (1-\lambda) \lambda \int_{\mathbb{P}_1(\cY)} \int_{\cY} G^*(Q_\lambda,y) \, \mathrm{d}p(y)\, \mathrm{d} (\underbrace{\widetilde{Q}  - Q}_{= - (Q - \widetilde{Q})})(p) \\
                &= G(Q_\lambda).
        \end{align*}
        Thus, $G$ is concave according to \cref{def:concavity_supertangent}.
        
        If $G^*$ is a strict supertangent, then all inequalities above are strict if $Q \neq \widetilde{Q}$ and $\lambda \in(0,1)$ and consequently $G$ is strictly concave in this case.
\end{proof}

\subsection{Proof of \cref{theorem_level_2_scoring_rule_characterization_2}}

\begin{proof}
	Suppose $S_2$ is order sensitive. 
	Since $ \mathbb{P}_2(\cY) $ is convex, we can represent any element $\hat Q$ by a suitable convex combination of a target second-order distribution $Q$ and another suitable second-order distribution $\widetilde{Q}.$ 
	Formally, for all $\hat Q,Q$ there exist $\lambda\in[0,1]$ and $\widetilde{Q}$ such that $\hat Q = (1-\lambda)\widetilde{Q} + \lambda Q.$
	Thus, 
	\begin{align*}
		S_2(\hat Q, Q) = 	S_2((1-\lambda)\widetilde{Q} + \lambda Q, Q) \geq S_2(Q,Q),
	\end{align*}
	which implies that $S_2$ is proper.
	This inequality is strict if $S_2$ is strictly order sensitive and $\hat Q \neq Q$ implying that $S_2$ is strictly proper.

	Now, assume that $S_2$ is proper.
	Note that any scoring-rule $S_2$ is (convex) linear in its second argument: For all $\lambda \in [0,1]$ and $Q,Q',\hat Q \in \mathbb{P}_2(\cY)$ it holds that
	\begin{align*}
		S_2(\hat Q, &\lambda Q + (1-\lambda) Q') \\
		&= \E_{p \sim \lambda Q + (1-\lambda) Q'} \big[\E_{Y \sim p}[ L_2(\hat Q, Y) ]\big] \\
		&= \int_{\mathbb{P}_1(\cY)} \int_{\cY} L_2(\hat Q, y)  \, \mathrm{d}p(y)\, \mathrm{d}(\lambda Q + (1-\lambda) Q')(p) \\
		&= \lambda \int_{\mathbb{P}_1(\cY)} \int_{\cY} L_2(\hat Q, y)  \, \mathrm{d}p(y)\, \mathrm{d}  Q(p)  \\
		&\quad + (1-\lambda) \int_{\mathbb{P}_1(\cY)} \int_{\cY} L_2(\hat Q, y)  \, \mathrm{d}p(y)\, \mathrm{d}  Q'(p) \\
		&= \lambda S_2(\hat Q,   Q  ) + (1-\lambda)  S_2(\hat Q,   Q'  ).
	\end{align*}
	Order sensitivity of a scoring rule holds if for all $\lambda\in[0,1]$ and $Q,Q' \in \mathbb{P}_2(\cY)$ 
	\begin{align} \label{equivalent_order_sensitivity_statement}
		S_2(Q',Q) - S_2(  (1-\lambda)Q' + \lambda Q, Q ) \geq 0.
	\end{align}
	This follows by propriety of $S_2:$ Abbreviate $\widetilde{Q} = (1-\lambda)Q' + \lambda Q$ and note that by convex linearity in the second argument
	\begin{align*}
            &(1-\lambda) S_2(Q',Q') + \lambda S_2(Q',Q) = S_2(Q',\widetilde{Q})
            \geq  S_2(\widetilde{Q},\widetilde{Q}) = (1-\lambda) S_2(\widetilde{Q},Q') + \lambda S_2(\widetilde{Q},Q),
	\end{align*}
	which is equivalent to
	\begin{align*}
		S_2(Q',Q) &- S_2(  (1-\lambda)Q' + \lambda Q, Q ) 
		\geq \frac{(1-\lambda)}{\lambda} \Big( S_2(\widetilde{Q}, Q') - S_2(Q',Q')   \Big).
	\end{align*}
	The right-hand side is non-negative since $S_2$ is proper, which implies \eqref{equivalent_order_sensitivity_statement}.
	Finally, if $S_2$ is strictly proper, then \eqref{equivalent_order_sensitivity_statement} holds with strict inequality implying strict order sensitivity of $S_2.$ 
\end{proof}

\end{document}